\documentclass[letterpaper]{article} 
\usepackage[]{aaai2026}  

\usepackage{booktabs}

\usepackage[english]{babel}

\usepackage{graphicx}

\usepackage{bm}
\usepackage{textcomp}
\usepackage{amssymb}
\usepackage{amsmath}
\usepackage{amsthm}
\usepackage{amsfonts}
\usepackage{mathtools}
\usepackage{pifont}

\usepackage{algorithm}
\usepackage{algpseudocode}

\usepackage{textcomp}
\usepackage{multirow}
\usepackage{subcaption}

\usepackage{microtype}
\usepackage{mdframed}

\usepackage{tabularx, colortbl, xcolor}
\definecolor{4A7BB7}{HTML}{4A7BB7}  
\definecolor{91C4DE}{HTML}{91C4DE}  
\definecolor{FED081}{HTML}{FED081}  
\definecolor{F7834D}{HTML}{F7834D}  
\definecolor{8ECFC9}{HTML}{8ECFC9}  
\definecolor{FFBE7A}{HTML}{FFBE7A}  
\definecolor{FA7F6F}{HTML}{FA7F6F}  
\definecolor{82B0D2}{HTML}{82B0D2}  

\definecolor{D62728}{HTML}{D62728}  
\definecolor{E13333}{HTML}{E13333}  
\definecolor{C63D52}{HTML}{C63D52}  
\definecolor{F67088}{HTML}{F67088}  
\definecolor{F77189}{HTML}{F77189}  
\definecolor{FFB3B2}{HTML}{FFB3B2}  

\definecolor{BF3E87}{HTML}{BF3E87}  
\definecolor{F584C5}{HTML}{F584C5}  
\definecolor{E377C2}{HTML}{E377C2}  
\definecolor{9467BD}{HTML}{9467BD}  
\definecolor{7B7AB7}{HTML}{7b7ab7}  
\definecolor{1B2278}{HTML}{1B2278}  
\definecolor{4E518B}{HTML}{4E518B}  

\definecolor{005397}{HTML}{005397}  
\definecolor{023047}{HTML}{023047}  
\definecolor{1C3168}{HTML}{1C3168}  
\definecolor{327BB7}{HTML}{327bb7}  
\definecolor{1F77B4}{HTML}{1F77B4}  
\definecolor{3BA3EC}{HTML}{3BA3EC}  
\definecolor{90C9E6}{HTML}{90C9E6}  
\definecolor{94C3E1}{HTML}{94C3E1}  
\definecolor{A2B1C3}{HTML}{A2B1C3}  

\definecolor{0192A8}{HTML}{0192A8}  
\definecolor{17BECF}{HTML}{17BECF}  
\definecolor{219EBC}{HTML}{219EBC}  
\definecolor{36ADA4}{HTML}{36ADA4}  
\definecolor{2CA02C}{HTML}{2CA02C}  
\definecolor{AEDB89}{HTML}{AEDB89}  
\definecolor{B3FEAE}{HTML}{B3FEAE}  

\definecolor{BCBD22}{HTML}{BCBD22}  
\definecolor{C99E31}{HTML}{C99E31}  
\definecolor{F0CF61}{HTML}{F0CF61}  
\definecolor{BC9F48}{HTML}{BC9F48}  

\definecolor{8C564B}{HTML}{8C564B}  
\definecolor{BEA1A5}{HTML}{BEA1A5}  

\definecolor{BFB5D7}{HTML}{BFB5D7}  
\definecolor{FFD8B0}{HTML}{FFD8B0}  
\definecolor{F3EEE3}{HTML}{F3EEE3}  
\definecolor{D9D9D9}{HTML}{D9D9D9}  

\definecolor{000000}{HTML}{000000}  
\definecolor{333333}{HTML}{333333}  

\usepackage[capitalize,noabbrev]{cleveref}

\theoremstyle{plain}
\newtheorem{theorem}{Theorem}[section]

\newtheorem{lemma}[theorem]{Lemma}
\newtheorem{corollary}[theorem]{Corollary}
\theoremstyle{definition}

\theoremstyle{remark}
\newtheorem{example}{Example}[]

\usepackage{braket}
\newcommand{\norm}[1]{\left\lVert#1\right\rVert}
\newcommand{\abs}[1]{\left\lvert#1\right\rvert}

\usepackage{stmaryrd}
\newcommand{\iverson}[1]{\left\llbracket#1\right\rrbracket}

\usepackage{tikz}
\usetikzlibrary{trees}
\usepackage{tikz-dependency}
\usepackage{tikz-3dplot}
\usetikzlibrary{chains,scopes}
\usetikzlibrary{arrows.meta,automata,positioning}
\usetikzlibrary{shapes.geometric, arrows}
\usepackage{pgfplots}

\usepgfplotslibrary{fillbetween}

\pgfplotsset{
  every axis/.append style = {thick},
  tick style = {thick,black},
  /tikz/normal shift/.code 2 args = {%
    \pgftransformshift{%
        \pgfpointscale{#2}{\pgfplotspointouternormalvectorofticklabelaxis{#1}}%
    }%
  },%
  shift/.style = {
    tick align        = outside,
    scaled ticks      = false,
    enlargelimits     = false,
    ticklabel shift   = {#1},
    axis lines*       = left,
    xtick style       = {normal shift={x}{#1}},
    ytick style       = {normal shift={y}{#1}},
    x axis line style = {normal shift={x}{#1}},
    y axis line style = {normal shift={y}{#1}},
  },
  shift/.default = 10pt,
  shift3d/.style = {
    shift=#1,
    ztick style       = {normal shift={z}{#1}},
    z axis line style = {normal shift={z}{#1}},
  },
  shift3d/.default = 10pt,
}

\tikzstyle{startstop} = [rectangle, rounded corners, minimum width=1cm, minimum height=0.5cm,text centered, draw=black]
\tikzstyle{io} = [trapezium, trapezium left angle=70, trapezium right angle=110, minimum width=2cm, minimum height=0.63cm, text centered, draw=black]
\tikzstyle{process} = [rectangle, minimum width=2cm, minimum height=0.5cm, text centered,draw=black]
\tikzstyle{process_n} = [rectangle, minimum width=1.5cm, minimum height=0.5cm, text centered,text width=1.5cm, draw=black]
\tikzstyle{decision} = [diamond, minimum width=1.2cm, minimum height=0.5cm, text width=1.5cm, text centered, aspect=1.5, draw=black]
\tikzstyle{arrow} = [->,>=stealth]

\usepackage[most]{tcolorbox}
\newtcolorbox[auto counter]{summary}[1][]{title={\bfseries Summary~\thetcbcounter},enhanced,
  coltitle=black,
  colback=white,
  top=0.3in,
  attach boxed title to top left=
  {xshift=1.5em,yshift=-\tcboxedtitleheight/2},boxrule=0.5pt,  sharp corners, fonttitle=\bfseries,boxed title style={size=small,colback=white,colframe=white},#1}
\usepackage{times}  
\usepackage{helvet}  
\usepackage{courier}  
\usepackage[hyphens]{url}  
\usepackage{graphicx} 
\usepackage{natbib}  
\usepackage{caption} 
\usepackage{newfloat}
\usepackage{listings}
\DeclareCaptionStyle{ruled}{labelfont=normalfont,labelsep=colon,strut=off} 
\floatstyle{ruled}
\newfloat{listing}{tb}{lst}{}
\floatname{listing}{Listing}
\title{Towards Provably Unlearnable Examples via Bayes Error Optimization}
\author{
    Ruihan Zhang, Jun Sun, Ee-Peng Lim, Peixin Zhang\thanks{Corresponding author.}
}
\affiliations{
    \textsuperscript{\rm 1}Singapore Management University\\

    81 Victoria St, Singapore, Singapore 188065\\
    pxzhang@smu.edu.sg
}

\usepackage{bibentry}

\begin{document}

\maketitle

\begin{abstract}
The recent success of machine learning models, especially large-scale classifiers and language models, relies heavily on training with massive data. These data are often collected from online sources. This raises serious concerns about the protection of user data, as individuals may not have given consent for their data to be used in training. To address this concern, recent studies introduce the concept of unlearnable examples, i.e., data instances that appear natural but are intentionally altered to prevent models from effectively learning from them. While existing methods demonstrate empirical effectiveness, they typically rely on heuristic trials and lack formal guarantees. Besides, when unlearnable examples are mixed with clean data, as is often the case in practice, their unlearnability disappears. In this work, we propose a novel approach to constructing unlearnable examples by systematically maximising the Bayes error, a measurement of irreducible classification error. We develop an optimisation-based approach and provide an efficient solution using projected gradient ascent. Our method provably increases the Bayes error and remains effective when the unlearning examples are mixed with clean samples. Experimental results across multiple datasets and model architectures are consistent with our theoretical analysis and show that our approach can restrict data learnability, effectively in practice. 

\end{abstract}


\section{Introduction}

In recent years, machine learning models have demonstrated remarkable performance across a wide range of tasks, driven by large-scale datasets and powerful computational resources~\cite{kaplan2020scaling,he2016deep}. However, this data-centric mechanism also raises serious concerns regarding intellectual property and control over data usage~\cite{yeom2018privacy,huang2021unlearnable}. Once data is shared or collected, often without users' full awareness, it becomes extremely difficult to restrict its subsequent usage or to prevent it from being exploited by unintended parties. For example, social media users are happy to share personal photos on blogs, yet still wish to prevent unauthorised third parties from leveraging these images for model training.

Motivated by the need for data protection at the source, recent studies have introduced the concept of unlearnable examples, a proactive defence mechanism that applies imperceptible perturbations to clean data, rendering it unlearnable by machine learning models~\cite{huang2021unlearnable}. Returning to the earlier example, a user can apply subtle, human-imperceptible perturbations to their photos before sharing them, thereby ensuring that any model trained on these perturbed images performs poorly when attempting to recognise or analyse the individual's face in unaltered images~\cite{wen2023adversarial}. This technique not only enables data owners to regain control over their data but also introduces a novel perspective in machine learning by designing data that actively resists learning.

Prior works for creating unlearnable examples take two main directions. One line of research perturbs data to increase training loss, aiming to obstruct the model from extracting useful patterns~\cite{fowl2021adversarial}. Another line encourages overfitting by minimising training error, thereby causing models to overfit non-generalisable features~\cite{fu2022robust}. While existing methods show empirical effectiveness, they typically require altering the entire training set. In practice, however, unlearnable examples are mixed with clean data from diverse sources~\cite{liu2023reliable,nakashima2022can,seddik2022neural}. In such cases, for example, training a model on unlearnable examples (CIFAR-10~\cite{krizhevsky2009learning}) may yield 20.50\% test accuracy. But when these unlearnable examples are mixed with the same number of clean examples for training, the test accuracy can be boosted to 92.51\%. This greatly limits the effectiveness of existing methods~\cite{huang2021unlearnable}.

In addition, a formal guarantee is underrated in existing methods, i.e., these strategies are heuristic instead of from a systematic derivation of unlearnability. This absence of theoretical grounding poses risks to interpretability and reliability, especially in safety-critical settings~\cite{wu2018beyond}. Therefore, there is a pressing need for a clear framework that not only provides theoretical guarantees for unlearnability but also remains effective under realistic conditions where clean and unlearnable data coexist.

In this work, we set out from a statistical learning concept of Bayes error, which is the classification error rate of the Maximum A Posterior rule~\cite{fukunaga1990introduction}. Essentially, Bayes error quantifies the inherent difficulty of classification with a given data distribution. Here, we leverage Bayes error as an effective measurement of data unlearnability, i.e., higher Bayes error means more unlearnable. Then, by systematically increasing the Bayes error through constrained perturbations, we can consequently ensure that the perturbed examples become provably harder to learn from, regardless of the training algorithm. Furthermore, this method remains effective even when unlearnable examples are mixed with clean samples.

To increase the Bayes error, we set up an optimisation problem to maximise the estimated Bayes error under norm-bounded constraints (to maintain the data quality). When solving this optimisation problem, we present an efficient solution with projected gradient ascent~\cite{bertsekas2003goldstein}. We show through a theoretical analysis that the Bayes error is guaranteed to increase. We then evaluate our method and implementation through extensive empirical studies. Experiment results demonstrate that our approach consistently leads to a test accuracy drop. For example, on a CIFAR-10 training set with 50\% clean and 50\% unlearnable examples, training on just the clean half gives 91.16\% test accuracy. Adding unlearnable examples from existing methods increases accuracy to 92.51\%, which defeats the purpose of being unlearnable. In contrast, using ours drops it to 69.68\%. This shows our method effectively produces unlearnable examples. Our code is available ~\cite{code}.

\section{Preliminaries and Problem Definition}

We first describe the machine learning background, including distributions, learning algorithms, and the concept of unlearnable examples. Then, we define our research problem.

\paragraph{Notations in machine learning}
Let $D$ denote a joint probability distribution for random variables $\mathbf{x}\in\mathbb{X}$ and $\mathbf{y}\in\mathbb{Y}$, such that for some positive integer $n$, a sample from $D$ is captured by $\left\{ (x_1, y_1), (x_2, y_2), \ldots, (x_n, y_n)\right\}$  or $\{(x_i, y_i)\}_{i=1}^n \sim D$, where $n$ is the sample size.

When it comes to machine learning regarding $D$, we consider $\mathbb{X}$ an input feature space, and consider $\mathbb{Y}$ a label space. Further, we would denote a model (learner) function $h: \mathbb{X} \to \mathbb{Y}$ that predicts outputs $h(x)\in \mathbb{Y}$ based on a (possibly high-dimensional) input point $x\in\mathbb{X}$. The quality of a model can be measured through $\operatorname{E}_{\mathbf{x}, \mathbf{y}\sim D}\left[\ell (h, \mathbf{x}, \mathbf{y}) \right]$ with a problem-dependent loss function $\ell(h, x, y)$~\cite{zhang2004statistical}.

\paragraph{Unlearnable examples}

In this work, we explore a scenario in which some of the training data needs to be protected from being learned by the model. A defender can read the original training dataset $\{(x_i, y_i)\}_{i=1}^n\sim D$, and can apply small perturbations $\Delta x_i$ to these samples. Typically, we have $\norm{\Delta x_i}_p\le \epsilon$ ($\epsilon>0$ caps the perturbation range within an $L^p$-norm) such that the data quality is not affected for human perception. This results in dataset $\{(x_i', y_i)\}_{i=1}^n$, where $\forall i\in 1, \ldots, n,  x_i' = x_i + \Delta x_i$. The goal is that when an arbitrary model trains on $\{(x_i', y_i)\}_{i=1}^n$ (and possibly with additional data other than \( \{(x_i, y_i)\}_{i=1}^n \)), the resultant model performs badly in testing (or at least worse that a model trained without the perturbed examples). Formally, let $h$ be a model  selected by arbitrary learning algorithm $\Gamma(\cdot)$ such that $h=\Gamma(x_1 + \Delta x_1, y_1, \ldots, x_n + \Delta x_n, y_n)$, and this goal could be expressed as
\begin{equation}
\label{eq:goal}
\begin{aligned}
    \max &\min_{h\in \mathbb{H}} \operatorname{E}_{\mathbf{x}, \mathbf{y}\sim D}\left[\ell (h, \mathbf{x}, \mathbf{y}) \right],\\
    \text{s.t.}&\quad \forall i\in 1, 2, \ldots, n, \quad \norm{\Delta x_i}_p\le \epsilon
\end{aligned}
\end{equation}
where $\mathbb{H}$ is the space of possible models that can be selected. Optimisation problem~(\ref{eq:goal}) represents a bilevel optimisation formulation, i.e., the minimisation here trains $h$, and maximisation manipulates $\Delta x_i$ to prevent trained $h$ from performing well in testing.

In practice, instead of solving optimisation problem~(\ref{eq:goal}), two alternative optimisation problems are tackled instead, as shown in optimisation problem~(\ref{eq:maxmin}) and optimisation problem~(\ref{eq:minmin}).
\begin{equation}
\label{eq:maxmin}
    \max_{\Delta x_i} \min_{h\in \set{\mathbb{X}\to\mathbb{Y}}} \frac{1}{n} \sum_{i=1}^n \ell_{\text{train}}(h, x_i + \Delta x_i, y_i)
\end{equation}
\begin{equation}
\label{eq:minmin}
    \min_{\Delta x_i} \min_{h\in \set{\mathbb{X}\to\mathbb{Y}}} \frac{1}{n} \sum_{i=1}^n \ell_{\text{train}}(h, x_i + \Delta x_i, y_i)
\end{equation}
The intuition behind optimisation problem~(\ref{eq:maxmin}) is that if the training loss remains high and the model fails to fit the training data, then naturally not much information has been transferred to the model. Conversely, for optimisation problem~(\ref{eq:minmin}), the idea is that if the training loss is very low and the model overfits the training data, then what has been learned may not generalise well.

\subsection*{Problem definition}

Next, we define our problem. Although existing methods are shown to be relatively effective in some settings, they suffer from two main shortcomings. First, and most importantly, they work only if the entire training dataset is are unlearnable example, which is hardly practical. Second, these strategies do not come with a formal guarantee for unlearnability growth. To address the above-mentioned shortcoming, we thus aim to develop a theoretically grounded method such that it works even if only a small portion of the data is perturbed to be unlearnable.

\section{Unlearnable Example Construction}

We propose a novel approach for constructing unlearnable examples. We begin by formalising data unlearnability based on Bayes error, then introduce a differentiable estimate from samples. We further design a method to provably increase the Bayes error to hinder learning, and discuss whether such an approach works in the practical scenario where clean and unlearnable data coexist.

\subsection{Bayes error as unlearnability measurement}

We focus on the classification problem in this work. In such a setting, the Bayes error represents the minimal inevitable error of any classifier. It is defined as the expected error probability of the maximum a posteriori rule. Formally, given a distribution $D$ over $\mathbb{X}\times\mathbb{Y}$, the Bayes error (rate) of $D$ can be expressed~\cite{fukunaga1990introduction,garber1988bounds} as: 
\begin{equation}
\label{eq:bayeserror1}
    \operatorname{E}_{\mathbf{x} \sim D_\mathbf{x}}\left[1 - \max_c p(\mathbf{y}=c|\mathbf{x})\right]
\end{equation}

As shown in the optimisation problem~(\ref{eq:goal}), a model attempts to learn from data by minimising the prediction error, which is the objective that the defender aims to counteract by increasing this error. Since the Bayes error represents the minimal inevitable error of any classifier, it provides a fundamental upper bound on the wild model’s performance. Therefore, by increasing the Bayes error, the defender effectively lower the upper bound of the model's performance, which equivalently increases the unlearnability of the data, thereby enhancing the unlearnability of the data.

\subsection{Increasing Bayes error given finite samples. }

While Bayes error computation is straightforward with a known distribution (Eq.~(\ref{eq:bayeserror1})), it is not clear how to compute Bayes error from samples. To this end, we first present how to approximate Bayes error from samples, and then propose a method to systematically increase the Bayes error.

\paragraph{Estimating Bayes error through local posterior averaging}
Given only sampled data $\{(x_{i}, y_{i})\}_{i=1}^n \sim D$, rather than the distribution $D$, we do not have access to the true posterior $p_D(\textnormal{y} \mid \mathbf{x})$. The posterior may appear quantifiable through the following expression.
\begin{equation}
\label{eq:posterior_trivial}
    \hat{p}_D(\textnormal{y} \mid \mathbf{x}=x) = \frac{\sum_j  \iverson{(x_i = x) \land (\textnormal{y} = y_j)}}{\sum_i \iverson{x_i = x} }
\end{equation}
Eq.~(\ref{eq:posterior_trivial}) measures the frequency of each outcome. Yet, it is likely that some outcomes are never observed, e.g., when the $\mathbb{X}$ space is continuous, making the Eq.~(\ref{eq:posterior_trivial}) undefined. To address this issue, we present \cref{alg:posterior} to estimate Bayes error, where the similarity function $s$ can be defined according to the joint distribution pattern, e.g., a Gaussian kernel function
\begin{equation}
    s(x_1, x_2) = \exp(-\|x_1 - x_2\|^2/2\sigma^2).
\end{equation}

\begin{algorithm}[ht]
\caption{Local posterior averaging \label{alg:posterior}}
\begin{algorithmic}[1]
\Require Sampled data $\{(x_i, y_i)\}_{i=1}^n$, $\mathbb{Y}$ space, similarity function $s: \mathbb{X}\times\mathbb{X}\to \mathbb{R}_{\ge 0}$
\Ensure Posterior values $\{\hat{p}_D(\textnormal{y} \mid \mathbf{x}=x_i)\}_{i=1}^n$

\For{$i\gets 1, 2, \ldots, n$}
    \For{$c \in  \mathbb{Y}$}
    \State $\hat{p}_D(\textnormal{y}=c \mid \mathbf{x}=x_i) \gets \frac{\sum_{j=1, j\neq i}^n \iverson{y_i = c} \cdot s(x_j, x_i)}{\sum_{k=1, k\neq i}^n s(x_k, x_i)}$
    \EndFor
\EndFor
\end{algorithmic}
\end{algorithm}

Essentially, \cref{alg:posterior} takes sampled data as input and for each sampled $x$, it provides the posterior at that $x$. This then provides the sampled posterior values, and therefore the sampled maximum posteriors. Bayes error can thus be estimated as
\begin{equation}
\label{eq:bayes_estimate}
    \hat{\beta}_D = 1 - \frac{1}{n} \sum_i^n \max_c\hat{p}_D(\textnormal{y}=c \mid \mathbf{x}=x_i),
\end{equation}
which is the fixed-sample form of Eq.~(\ref{eq:bayeserror1}).
\begin{figure}
\centering
\begin{subfigure}[t]{0.5\linewidth}
\centering
\includegraphics[width=\linewidth]{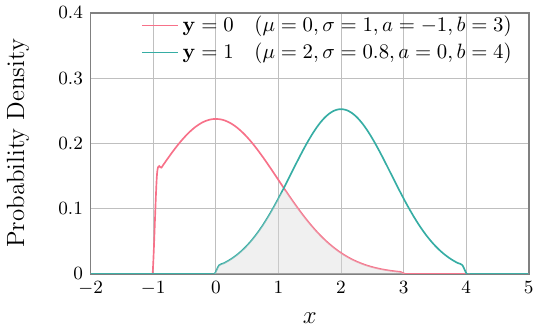}
\end{subfigure}
\begin{subfigure}[t]{0.437\linewidth}
\centering
\includegraphics[width=\linewidth]{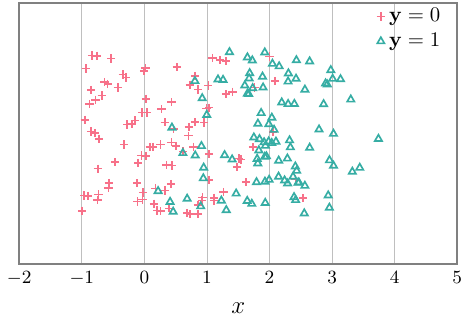}
\end{subfigure}

\caption{A joint distribution of two truncated normal distributions. (Left) The density function of this distribution. The grey-shaded area represents the Bayes error. (Right) A sample is drawn from this distribution. Note that the samples are randomly placed in the vertical direction.}
\label{fig:runex1}
\end{figure}
\begin{example}
Consider a binary classification problem in $\mathbb{R}$ with two truncated normal distributions, as illustrated in the \cref{fig:runex1}. The Bayes error of such a distribution is 0.1427, which can be analytically calculated from the probability density function (PDF) of truncated normal distributions. \cref{fig:runex1} (Right) illustrates sampled data drawn from this distribution. When estimating Bayes error from samples, a naive estimation like Eq.~(\ref{eq:posterior_trivial}) yields an undefined Bayes error, because there are no overlapping sampled outcomes. In contrast, when applying \cref{alg:posterior} and adopting the radial basis function~\cite{chang2010training}  for similarity, we get 0.1426 as an estimate of Bayes error, which is within a $\pm 0.1\%$ range of the analytic result.
\end{example}

In the following, we establish that the presented Bayes error estimate through local posterior averaging well approximates the analytic result.

\begin{lemma}
\label{lemma:estimate}
Let $s$ be a symmetric, non-negative similarity function satisfying the following conditions,
\begin{equation}
\begin{aligned}
    &\forall x'. \quad \int s(x, x')\, dx = 1,\\
    &\forall \alpha > 0. \quad\int_{\|x - x'\| > \alpha} s(x, x') \, dx \to 0 \quad \text{as}~~ n \to \infty,\\
    &\exists \kappa > 0.\quad \forall x, x', n.\quad 0 \le s(x, x') \le \kappa.
\end{aligned}
\end{equation} 
The posterior estimate from \cref{alg:posterior} converges to the true posterior, i.e.,
\begin{equation}
\begin{aligned}
    \operatorname{E}[|\hat{p}(\mathbf{y} = c \mid \mathbf{x}) - p(\mathbf{y} = c \mid \mathbf{x})|] \le 
    C_2\cdot\sigma^2 + \frac{C_1}{\sqrt{n} \sigma^{d/2}} .
\end{aligned}
\end{equation}
where $C_1$ and $C_2$ are constants.
\end{lemma}
Proof for \cref{lemma:estimate}, which leverages the bias and variance derivation for the Nadaraya-Watson estimator~\cite{nadaraya1964estimating,watson1964smooth}, is provided in supplementary material.

\begin{theorem}
\label{thm:estimate}

With the posterior estimate from \cref{alg:posterior}, the Bayes error estimate converges to the true Bayes error at the rate 
$2\cdot C_2\cdot\abs{\mathbb{Y}} \sigma^2 + \frac{2\cdot C_1\cdot\abs{\mathbb{Y}}}{\sqrt{n \sigma^d}}$.
\end{theorem}

\begin{proof}

From Eq.~(\ref{eq:bayeserror1}), the excess classification error $\abs{\hat{\beta}_D - \beta_D}$ can be expressed as
\begin{equation}
\begin{aligned}
    & \operatorname{E}_{\mathbf{x}}\left[1 - \max_c \hat{p}(\mathbf{y}=c|\mathbf{x})\right] -\operatorname{E}_{\mathbf{x} }\left[1 - \max_c p(\mathbf{y}=c|\mathbf{x})\right]\\
    & \le 2\cdot \operatorname{E}_{\mathbf{x} }\left[ \sum_c \left| p(\mathbf{y}=c|\mathbf{x}) - \hat{p}(\mathbf{y}=c|\mathbf{x}) \right|\right]
\end{aligned}
\end{equation}
Now apply the posterior estimation error bound  $ C_2 \sigma^2 + \frac{C_1}{\sqrt{n \sigma^d}}$ in \cref{lemma:estimate}, where $C_1$ and $C_2$ are constants. Then, we get $\abs{\hat{\beta}_D - \beta_D} \le  2\cdot C_2\cdot\abs{\mathbb{Y}} \sigma^2 + \frac{2\cdot C_1\cdot\abs{\mathbb{Y}}}{\sqrt{n \sigma^d}}$, still within the same asymptotic order.
\end{proof}

\paragraph{Bayes error maximisation under $L^p$-norm perturbation constraint}

Given sampled data $\{(x_{i}, y_{i})\}_{i=1}^n \sim D$, we would like to find a distribution $D'$ close to $D$ such that we have $\hat{\beta}_{D'} > \hat{\beta}_{D}$. Specifically, let $\{(x_{i}', y_{i}')\}_{i=1}^n \sim D'$ be sampled data from $D'$, and an $L^p$ closeness constraint between $D'$ and $D$ can be expressed as
\begin{equation}
    \forall i\in [1, n]\cap \mathbb{N}.\quad \left(y_i' = y_i\right) \land \left(\norm{x_i' - x_i}_p \le \epsilon\right)
\end{equation}
where $\epsilon > 0$ is the $x$ perturbation range. Since $y_i'$ always equals $y_i$ for all $i$, we use $y_i$ whenever $y_i'$ is needed. Applying Eq.~(\ref{eq:bayes_estimate}), we can estimate Bayes error of $D'$ as
\begin{equation}
\label{eq:bayes_estimate_perturbed}
    \widehat{\beta}_{D'} = \frac{1}{n} \sum_{i=1}^n \left(1 - \max_c \hat{p}_{D'}(\textnormal{y} = c \mid \mathbf{x}=x_i')\right).
\end{equation}

We are now ready to define our constrained optimisation problem of finding $D'$ as follows. Essentially, we aim to find  $x_1, \dots, x_n$ that maximise the Bayes error estimate, subject to an $L^p$ norm constraint.
\begin{equation}
\label{eq:optimisation}
\begin{aligned}
    \max_{\{x_i\}_{i=1}^n} \quad & \frac{1}{n} \sum_{i=1}^n \left(1 - \max_c \frac{\sum_{j=1, j\neq i}^n \iverson{y_i = c} \cdot s(x_j', x_i')}{\sum_{k=1, k\neq i}^n s(x_k', x_i')}\right) \\
    \text{s.t.} \quad & \|x_i' - x_i\|_p \le \epsilon, \quad \forall i = 1, 2, \dots, n
\end{aligned}
\end{equation}
Note that the objective function in Eq.~(\ref{eq:optimisation}) is an expanded form of Eq.~(\ref{eq:bayes_estimate_perturbed}) using \cref{alg:posterior}, and thus fully explicit given the original data $x_i, y_i$ and the similarity function $s$. This objective function is also computationally tractable and differentiable almost everywhere.

Next, we present the solution to the optimisation problem above. Direct gradient ascent is not applicable in constrained settings, as it does not guarantee that the variable in each iteration remains within the constrained set. To address this, we employ projected gradient ascent (PGA), a first-order iterative method that enforces feasibility via projection. PGA operates by performing a standard gradient ascent step followed by a projection onto the constrained set. The details of our approach are shown in~\cref{alg:pga}

\begin{algorithm}[ht]
\caption{Projected gradient ascent for Bayes error\label{alg:pga}}
\begin{algorithmic}[1]
\Require Sampled data $\{(x_i, y_i)\}_{i=1}^n$, $\mathbb{Y}$ space, similarity function $s: \mathbb{X}\times\mathbb{X}\to \mathbb{R}_{\ge 0}$, step size $\eta > 0$, $L^p$ constraint parameter $\epsilon$, maximum iterations $T$
\Ensure $\{(x_i'^{(T)}, y_i)\}_{i=1}^n$

\State Initialize $x_i'^{(0)} \gets \bm{0}$, $\forall i \in 1, 2, \ldots, n$
\For{$t \gets 0, 1, \ldots, T-1$}

    \For{$i = 1$ ... $n$}
        \State $\bm{g}_i^{(t)} \gets- \nabla_{x_i'}\left.\left( \sum_{i=1}^n \frac{\max_c \hat{p}_{D'}( c|x_i')}{n}\right)\right|_{\forall j. x_j' = x_j'^{(t)}}$
        \State $\bm{\delta}_i^{(t+1)} \gets x_i'^{(t)} + \eta \cdot \bm{g}_i^{(t)} -  x_i$
        \State $ \bm{\delta}_i^{(t+1)} \gets \arg\min_{\bm{\delta}', ~\|\bm{\delta}'\|_p \le \epsilon} \norm{\bm{\delta}' - \bm{\delta}_i^{(t+1)}}_2^2  $
        \State $x_i'^{(t+1)} \gets \bm{\delta}_i^{(t+1)} + x_i$
    \EndFor
\EndFor

\end{algorithmic}
\end{algorithm}

Essentially, each iteration of the gradient ascent algorithm requires computing gradients of the objective function (Bayes error estimate) with respect to all variable inputs $x_i'$ (line 4). Then, we use this gradient to shift the current $x_i'$ to the direction where the Bayes error grows (line 5). After that, we project the shift to the $L^p$ constraint to obtain a feasible next-step $x_i'$ (line 6). Line 6 in \cref{alg:pga} gives a general form expression for projection, and for $L^\infty$-norm, a closed-form projection can be expressed as $ \max(-\epsilon, \min(\bm{\delta}_i^{(t+1)}, \epsilon))$.

Note that gradient computation (Line 4 in \cref{alg:pga}) is a fundamental step for PGA to work. Eq.~(\ref{eq:gradient}) below shows the partial derivative of the objective function in Eq.~(\ref{eq:optimisation}) with respect to each $x_i$.

\begin{example}
\begin{figure}
    \centering
    \begin{subfigure}[t]{\linewidth}
\centering
\newcommand{\isIncluded}{}
\includegraphics[width=0.8\linewidth]{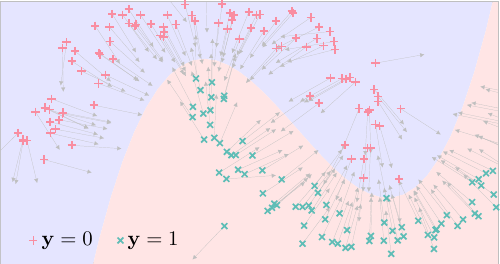}

\end{subfigure}

\caption{Perturbation of two-dimensional points.}
\label{fig:perturb}
\end{figure}
To illustrate the effect of perturbations, we present an example in \cref{fig:perturb}. Given sampled data from Moons (a simple Scikit-learn toy dataset to visualise classification~\cite{scikit-learn,chen2023evaluating}),  our goal is to perturb each point within an $L^\infty$ norm, e.g., $\norm{\Delta x_i}_\infty < 0.15$ in this example. Next, we apply the PGA procedure described earlier to obtain perturbations $\bm{\delta}_i^*$ for each sample, constrained by $\|\bm{\delta}_i\|_2 \leq \epsilon$, with $\epsilon = 0.15$.
Then, according to \cref{alg:pga}, we obtain the perturbed points. \cref{fig:perturb} illustrates the perturbation, as well as the perturbed points. We can then use \cref{alg:posterior} to find that the new Bayes error value is 0.1832, 28\% higher than the original one. Apart from the Bayes error growth, we can observe that while generally the movement of each point is towards the other class (increasing classification uncertainty), the travelled distance does not necessarily equal the perturbation constraint, i.e., 0.15 here. This suggests that in some cases, it is not that a larger shift is more optimal.

\end{example}

In the following, we establish that applying \cref{alg:pga} can lead to a guaranteed increase in the Bayes error.
\begin{lemma}
\label{lemma:pga}
Let $f:\mathbb{R}^d \to \mathbb{R}$ be a differentiable function with a \(\kappa\)-Lipschitz continuous gradient. Let \( \mathbb{C} \subseteq \mathbb{R}^d \) be a closed, convex set. Let the PGA update be defined as
\begin{equation}
    x^{(k+1)} = \Pi_{\mathbb{C}} \left( x^{(k)} + \eta \nabla f(x^{(k)}) \right)
\end{equation}
with \( x^{(0)} \in \mathbb{C} \), and step size \( \eta \in \left( 0, \frac{2}{\kappa} \right) \). Then for all \( k \geq 0 \),
\begin{equation}
f(x^{(k)}) \geq f(x^{(0)}).
\end{equation}
\end{lemma}
Proof for \cref{lemma:pga} is in supplementary materials. Next, we establish that our objective function satisfies the $\kappa$-Lipschitz condition.
\begin{theorem}
\label{thm:increase}
There exists a constant $\kappa$ such that for all $\eta \in (0, \frac{2}{\kappa})$, we have that the for any $\{(x_i', y_i')\}_{i=1}^n$ from \cref{alg:pga}, it is guaranteed that $\hat{\beta}_{D'}\ge \hat{\beta}_{D}$.
\end{theorem}

\begin{proof}
We consider the objective function in Eq.~(\ref{eq:optimisation}) as a function $f_1: \mathbb{X}^n\times\mathbb{Y}^n \to \mathbb{R}$, and the perturbation range as a set \( \mathbb{C} \subseteq \mathbb{X}^n \). First, we show that this objective function is a differentiable function with a \(\kappa\)-Lipschitz continuous gradient. Specifically, the partial derivative (of $f_1$) with respect to each variable can be expressed as Eq.~(\ref{eq:gradient}).
\begin{equation}
\label{eq:gradient}
\begin{aligned}
    &\nabla_{x_i'}\left(\frac{1}{n} \sum_{i=1}^n \left(1 - \max_c \frac{\sum_{j=1, j\neq i}^n \iverson{y_i = c} \cdot s(x_j', x_i')}{\sum_{k=1, k\neq i}^n s(x_k', x_i')}\right)\right)\\
    &= -\frac{1}{n}\sum_{j=1, j\neq i}^n  \left(\frac{1}{(\sum_k s(x_i', x_k'))^2} \cdot \max_c \sum_k ( \iverson{y_j = c} -\right.\\
    &\hspace{23mm}\iverson{y_k = c} ) s(x_i, x_k)+ \max_c \sum_k ( \iverson{y_i = c} - \\
    &\iverson{y_k = c} ) \left.\frac{1}{(\sum_k s(x_j', x_k'))^2} \cdot s(x_j, x_k)\right) \nabla_{x_i'} s(x_i', x_j')
\end{aligned}
\end{equation}
Observe that the function $s$ satisfies $\exists \kappa_1 > 0.\quad \forall x, x'.\quad 0 \le s(x, x') \le \kappa_1$. Thus, as long as it is not the case that $\forall i, j, s(x_i, x_j) = 0$, then we get that
\begin{equation}
    \left(\exists \kappa_2.~ \norm{\nabla_{x_i'} s(x_i', x_j') } \le \kappa_3\right) \to \left(\exists \kappa_3.~ \norm{\nabla_{x_i'} f_1 } \le \kappa_3\right).
\end{equation}
That is, if the similarity function is $\kappa_2$-Lipschitz, then the objective function of Eq.~(\ref{eq:gradient}) is $\kappa_3$-Lipschitz. Next, we show that the perturbation range $\mathbb{C}$ is a closed, convex set. Namely, for all orders greater than 0, $L^p$ norm is closed, and for all orders greater than or equal to 1, $L^p$ norm is convex. Therefore, applying \cref{lemma:pga}, we show that for all \( k \ge 0 \):
\begin{equation}
    f_1(x_1^{(k)},\ldots x_n^{(k)}, y_1, \ldots) \geq f_1(x_1^{(0)},\ldots x_n^{(0)}, y_1, \ldots).
\end{equation}
\end{proof}
\paragraph{Dealing with embedded inputs}

The previous discussion is based on a discrete representation of samples whose distance is easy to measure. As for complicated inputs, e.g., coloured images, we could first compute their embedding vector and then solve the Bayes error maximisation problem. Specifically, let $m:\mathbb{X}_\text{ori}\to \mathbb{X}_\text{emb}$ denote the embedding function from original input space $\mathbb{X}_\text{ori}$ to the embedding space $\mathbb{X}_\text{emb}$. In this case, their similarity is measured in the embedding space, and the constraint is still applied to the original input space, i.e., Eq.~(\ref{eq:optimisation}) becomes Eq.~(\ref{eq:optimisation-embed}) and $m$ needs to be known.
\begin{equation}
\label{eq:optimisation-embed}
\begin{aligned}
    \min_{\{x_i\}_{i=1}^n} ~~ & \sum_{i=1}^n \left(\max_c \frac{\sum_{j=1, j\neq i}^n \iverson{y_i = c} \cdot s(m(x_j'), m(x_i'))}{\sum_{k=1, k\neq i}^n s(m(x_k'), m(x_i'))}\right) \\
    \text{s.t.} ~~ & \|x_i' - x_i\|_p \le \epsilon, \quad \forall i = 1, 2, \dots, n
\end{aligned}
\end{equation}

\subsection{Mixing Clean and Unlearnable Examples}

In practice, models are unlikely to be trained solely on the perturbed unlearnable data. To model this situation, we adjust Eq.~(\ref{eq:optimisation}) by letting some perturbation be 0. That is, for some (but not all) $i$ in $1, 2, \ldots, n$, it is restricted that $x_i' = x_i$. Intuitively, this could be a case where the adversary could leverage publicly available data to augment the training, whereas the defender cannot modify all the public data. In the following, we show in \cref{cor:mix} that \cref{alg:pga} can still effectively increase Bayes error.

\begin{corollary}
    \label{cor:mix}
    Given $\{(x_i, y_i)\}_{i=1}^{n}$, let $\mathbb{I}_n\subset \mathbb{N}\cap [1, n], \mathbb{I}_n\neq \empty$ be a subset of all indices. Suppose for all indices $j$ in $\mathbb{I}_n$, it is restricted that $\Delta x_j = 0$. It follows that the remaining perturbations can increase Bayes error, i.e., the following condition is satisfied.
    \begin{equation}
    \begin{aligned}
        \left(\forall j \in \mathbb{I}_n.~\Delta x_j = 0 \right) \land\\
        \left(1 - \frac{1}{n} \sum_i^n \max_c\hat{p}_D(\textnormal{y}=c \mid \mathbf{x}=x_i')\right. \ge\\
        \left.1 - \frac{1}{n} \sum_i^n \max_c\hat{p}_D(\textnormal{y}=c \mid \mathbf{x}=x_i)\right)
    \end{aligned}
    \end{equation}
Proof for \cref{cor:mix} is in supplementary materials.
\end{corollary}

\section{Experiments}
\begin{figure*}[t]
    \centering
    \begin{subfigure}[t]{0.24\linewidth}
\centering
\includegraphics[width=\textwidth]{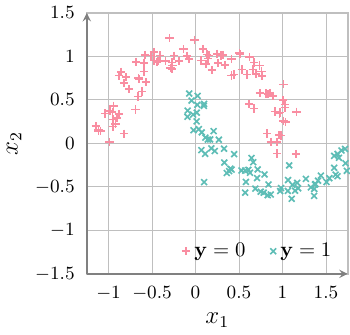}
\caption{Moons ($\hat{\beta} = 0.1434$)\label{fig:moons-original}}
\end{subfigure}
    \begin{subfigure}[t]{0.24\linewidth}
\centering
\includegraphics[width=\textwidth]{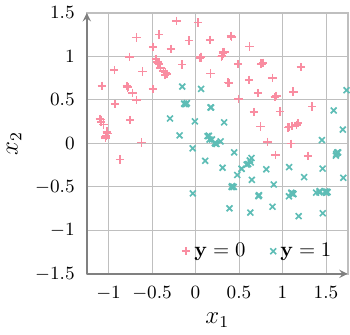}
\caption{Moons ($\hat{\beta} = 0.1888$)\label{fig:moons-unlearnable}}
\end{subfigure}
    \begin{subfigure}[t]{0.24\linewidth}
\centering
\includegraphics[width=\textwidth]{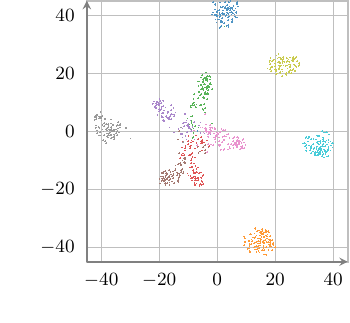}
\caption{CIFAR-10 ($\hat{\beta} = 0.0565$)\label{fig:cifar10-original}}
\end{subfigure}
    \begin{subfigure}[t]{0.24\linewidth}
\centering
\includegraphics[width=\textwidth]{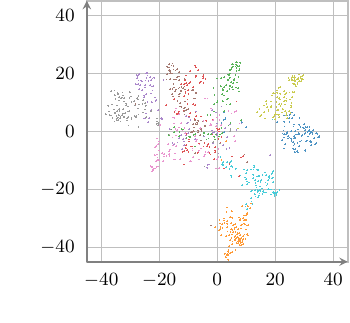}
\caption{CIFAR-10 ($\hat{\beta} = 0.2453$)\label{fig:cifar10-unlearnable}}
\end{subfigure}

\caption{Examples before and after perturbation}
\label{fig:rq1}
\end{figure*}
\begin{figure}[t]
    \centering
    \includegraphics[width=\linewidth]{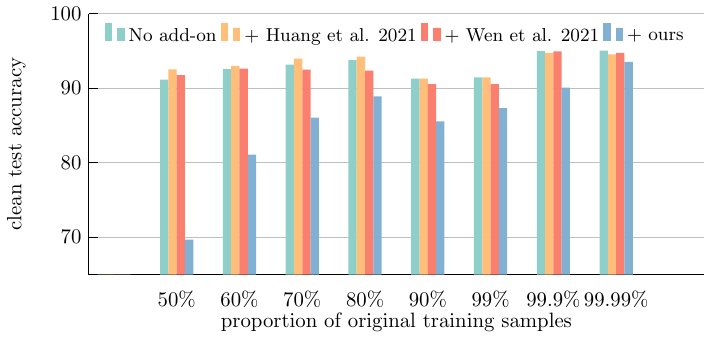}
    \caption{Test accuracy on each setting. ``No add-on'' means only training on part of the original training set.}
    \label{fig:rq2}
\end{figure}
We conduct experiments to evaluate the empirical effectiveness of the proposed method, aiming to answer the following research questions (RQs).

\begin{enumerate}
    \item  How does the proposed perturbation affect Bayes error? 
    \item Does training on our unlearnable examples lead to reduced test accuracy? What happens when only a subset of the training data is made unlearnable?
    \item How does our method perform against adaptive attacks such as adversarial training?
\end{enumerate}

\paragraph{Experimental settings}

We follow standard experimental settings as established in previous work~\cite{huang2021unlearnable,wen2023adversarial}, and evaluate our method on three image classification benchmarks: CIFAR-10, CIFAR-100~\cite{krizhevsky2009learning}, and Tiny ImageNet~\cite{le2015tiny}. Perturbations are constrained within an $L^\infty$-norm bound of $\epsilon = 8/255$ per image. For feature extraction, we adopt the commonly used ResNet-18 model~\cite{he2016deep,robey2022probabilistically,zhang2019theoretically,wang2019improving}. All models are trained for 100 epochs using stochastic gradient descent~\cite{robbins1951stochastic} with a learning rate of 0.1, momentum of 0.9, weight decay of $5\times 10^{-4}$, and cosine annealing learning rate scheduling~\cite{loshchilov2017sgdr}. For our optimisation implementation, we formulate an auto-differentiable Bayes error objective in PyTorch~\cite{paszke2019pytorch}, enabling gradient computation in the input space via back propagation. Additional experimental details are provided in the corresponding subsections.

\subsection{RQ1: The change in Bayes error}

We empirically study how the Bayes error changes after perturbing data inputs using our method. As an intuitive illustration, we first conduct the experiment on Moons (a simple toy dataset to visualise classification~\cite{scikit-learn,chen2023evaluating}), which consists of two interleaving half circles. A sample of 200 points is shown in \cref{fig:moons-original}. We then apply \cref{alg:pga} with $\epsilon = 0.25$ to construct unlearnable examples, resulting in the distribution shown in \cref{fig:moons-unlearnable}. The perturbed examples appear more scattered within each class, and examples from different classes become closer. This reduces inter-class separation, leading to a 24.3\% increase in Bayes error, from 0.1434 (original) to 0.1888 (unlearnable).

Next, we extend our experiment to high-dimensional image data. As shown in \cref{fig:cifar10-original}, the t-SNE visualization of the original CIFAR-10 training set exhibits well-separated class features. After applying \cref{alg:pga}, the resulting unlearnable dataset (visualized in \cref{fig:cifar10-unlearnable}) shows significant feature overlap across classes, making the data intuitively less learnable. Correspondingly, the Bayes error increases significantly from 0.0565 to 0.2453,  representing more than a 3.3x increase. These results confirm that our method effectively increases the Bayes error.

\subsection{RQ2: Effectiveness of unlearnable examples}

To evaluate the effectiveness of our unlearnable examples, we use them for training and then measure the test accuracy. Test accuracy is an empirical measurement of unlearnability. Intuitively, if the test accuracy remains high, the training data may not be sufficiently unlearnable. As a reference, training on the full original CIFAR-10 (training) set yields 95.12\% test accuracy.

We first train on our entire unlearnable set and observe a test accuracy of 28.12\% on CIFAR-10, showcasing an over 70\% accuracy drop from 95.12\%. Then, we explore mixed-data settings, e.g., replacing half of the original training samples with unlearnable ones results in 69.68\% test accuracy, with a 27\% drop from 95.12\%. To further isolate the effect of unlearnable examples, we train only on the original half set without any unlearnable data add-on, which achieves 91.16\% accuracy, over 30\% higher than 69.68\%. This suggests including unlearnable examples degrades model performance.

We iteratively adjust the proportion of original/unlearnable data in the mixture to reflect realistic scenarios where protected data forms a minority. As shown in \cref{fig:rq2}, we report test accuracy under varying ratios of kept original samples  (\# original / \# original and unlearnable) ratios. We include two representative baseline methods, one by minimising training loss and the other by maximising training loss. Specifically, \citet{huang2021unlearnable} introduce label-correlated noise to create shortcut features that hinder meaningful learning. In contrast, \citet{wen2023adversarial} identify the embedding centre of each class and push individual examples away from it, reducing intra-class coherence.

Compared with baseline methods, our method consistently induces greater accuracy drops, on average, 8\% more than \citet{huang2021unlearnable} and 9\% more than \citet{wen2023adversarial}. More importantly, our examples cause a consistent accuracy decline across all settings. In contrast, previous methods may occasionally act as extra training data and even improve accuracy when mixed in (e.g., 50\%~\cite{huang2021unlearnable}). These results demonstrate that our unlearnable examples are not only effective in full training scenarios but also degrade more when only used as a part.

\paragraph{Effectiveness on complex datasets}

\begin{figure}[t]
    \centering
    \includegraphics[width=\linewidth]{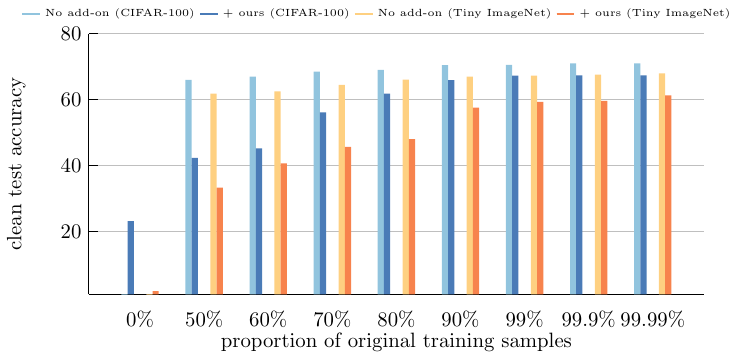}
    \caption{Test accuracy for complex datasets}
    \label{fig:change-dataset}
\end{figure}
\cref{fig:change-dataset} illustrates experimental results on more complex datasets, CIFAR-100 and Tiny ImageNet. Across all mixture ratios, including the case when fully training on unlearnable examples, our constructed examples consistently cause a noticeable drop in test accuracy. These results further confirm the general effectiveness of our method on larger and more challenging datasets.

\paragraph{Transferability to unseen model architectures}

\begin{table}[t]
\centering
\begin{tabular}{@{}>{\centering\arraybackslash}p{0.50\linewidth}|cc}
\toprule
Model architecture & clean & ours \\\midrule
ResNet-18    & 95.04 & 28.12 \\
ResNet-34    & 95.12 & 29.86 \\
VGG-19       & 91.70 & 18.14 \\
DenseNet-121 & 95.06 & 25.89 \\
MobileNet v2 & 93.30 & 30.54 \\ \bottomrule
\end{tabular}
\caption{Test accuracy when unlearnable examples are used in training with various architectures
\label{tab:change-architecture}
}
\end{table}

To transform image data into an embedding space, we use a feature extractor ResNet-18~\cite{he2016deep}. A natural question is whether the constructed unlearnable examples remain effective when a different model architecture is used for training. To investigate this, we follow \citet{wen2023adversarial} and evaluate four alternative architectures, with results summarised in \cref{tab:change-architecture}. Across all architectures, training on our CIFAR-10 unlearnable examples consistently yields lower accuracy compared to training on clean data. The average accuracy drop is 68\% ($\pm$10\%), closely matching the drop observed with ResNet-18 (28.12\%, $\pm$5\%), suggesting that the effectiveness of our examples generalises well across model architectures.

\subsection{RQ3: Resistance to adversarial training}
\begin{figure}[t]
    \centering
    \includegraphics[width=\linewidth]{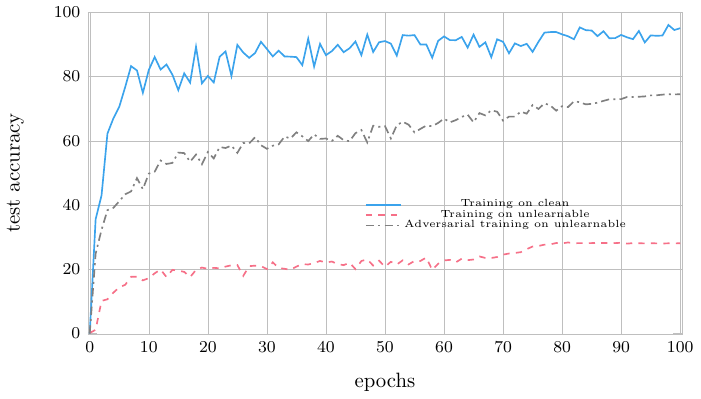}
    \caption{Test accuracy at each iteration}
    \label{fig:adaptive}
\end{figure}

Due to their potential public release, once constructed, unlearnable examples are no longer amendable by the defender. This may place them at a disadvantage, especially when it is known that the data has been intentionally perturbed. In such cases, countermeasures such as adversarial training may be applied to forcibly extract information from these examples. To evaluate the effectiveness of our approach in this scenario, we conduct experiments using PGD-based adversarial training~\cite{madry2017towards,wen2023adversarial}, a commonly adopted adaptive strategy, on our constructed unlearnable examples. \cref{fig:adaptive} presents model performance over training iterations. Compared to standard training on clean data, our unlearnable examples consistently lead to reduced accuracy. Notably, under adversarial training, the model achieves only 74.5\% test accuracy, which is 21.5\% lower than training on original examples. This substantial drop renders the resulting model largely unusable in practice.

\section{Related Work}
The idea of protecting data at the source has gained increasing attention due to the rising concerns over unauthorised data usage. \citet{huang2021unlearnable} and \citet{ren2022transferable} introduced the concept of unlearnable examples, where imperceptible perturbations are applied to make data inherently difficult to learn. Approaches to constructing unlearnable examples have typically followed two heuristics: (1) maximising the training loss to prevent models from fitting useful features~\cite{wen2023adversarial,fowl2021adversarial}, and (2) minimising it to encourage overfitting on non-generalizable patterns~\cite{huang2021unlearnable,fu2022robust}. There are also other approaches~\cite{liu2024game,liu2024stable,liu2023reliable,mu2025bayesian}, including directly optimising perturbations for specific model parameters~\cite{lu2023exploring} or target retraining scenarios using sharpness-aware techniques~\cite{he2023sharpness}. Such created unlearnable examples have been found to work ineffectively when mixed with clean data and lack formal guarantees. This work thus aims to tackle the problem of ensuring unlearnability even in the mixed data situation. 

Estimating the Bayes error of a given data distribution has long been an interesting topic~\cite{fukunaga1975k}. Related methods include estimating its lower or upper bounds, e.g. the Bhattacharyya distance~\cite{fukunaga1990introduction} or the Henze-Penrose divergence~\cite{berisha2015empirically,sekeh2020learning}. Alternatively, there are attempts to estimate the Bayes error using generative models~\cite{kingma2018glow,theisen2021evaluating} or real-world human annotation~\cite{renggli2021evaluating}. Indirect Bayes error estimation has also been adopted by quantity substitution, e.g., accuracy to robustness~\cite{zhang2024certified,zhang2024does}. Compared with these estimates, our Bayes error estimate is easier to compute, such as to model and improve unlearnability.

\section{Conclusion}
We propose a formal objective to increase unlearnability by increasing Bayes error, and a projected gradient ascent strategy to ensure its growth in practice. Experiments confirm our method consistently reduces test accuracy even if the training data is a mixture of clean and unlearnable ones. Overall, it offers an effective approach to user data protection. In the future, an interesting direction is to extend our framework to protect data from generative models, such as generative adversarial networks.

\section*{Acknowledgements}
This research is supported by the Ministry of Education, Singapore under its Academic Research Fund Tier 3 (Award ID: MOET32020-0004).


\end{document}